\documentclass{article}
\usepackage{ijcai22}

\usepackage{times}
\usepackage{soul}
\usepackage{url}
\usepackage[hidelinks]{hyperref}
\usepackage[utf8]{inputenc}
\usepackage[small]{caption}
\usepackage{graphicx}
\usepackage{amsmath}
\usepackage{amsthm}
\usepackage{booktabs}
\usepackage{algorithm}
\usepackage[noend]{algorithmic}
\usepackage{thmtools} 
\usepackage{thm-restate}
\usepackage{enumitem}

\usepackage{amssymb}

\usepackage{xcolor}
\usepackage{amssymb}
\usepackage{comment}
\renewcommand{\algorithmicrequire}{\textbf{Input:}}

\renewcommand{\and}{\mbox{ and }}
\newcommand{\inc}{\sim}

\newcommand{\AlternativeSet}{\mathcal{A}}
\newcommand{\TierFunction}{\gamma}
\newtheorem{proposition}{Proposition}
\newtheorem{example}{Example}

\title{Cautious Learning of Multiattribute Preferences}

\author{
Hugo Gilbert$^1$ \and
Mohamed Ouaguenouni$^2$ \and
Meltem Ozturk$^1$ \and
Olivier Spanjaard$^2$
\affiliations
$^1$Université Paris-Dauphine, Université PSL, CNRS, LAMSADE, 75016 Paris, France\\
$^2$Sorbonne Université, CNRS, LIP6, 75005, Paris, France
\emails
firstname.lastname@lamsade.dauphine.fr,\\
firstname.lastname@lip6.fr
}

\begin{document}

\maketitle

\begin{abstract}
This paper is dedicated to a cautious learning methodology for
predicting preferences between alternatives characterized by binary
attributes (formally, each alternative is seen as a subset of
attributes). By ``cautious'', we mean that
the model learned to represent the multi-attribute preferences is
general enough to be compatible with any strict weak order on the
alternatives, and that we allow ourselves not to predict some
preferences if the data collected are not compatible with a reliable
prediction. A predicted preference will be considered reliable if all
the simplest models (following Occam's razor principle) explaining the
training data agree on it. Predictions are based on an ordinal dominance relation between alternatives [Fishburn and LaValle, 1996]. The dominance relation relies on an uncertainty set encompassing the possible values of the parameters of the multi-attribute utility function. %Moreover, we interactively query a decision-maker, following a  strategy aiming at enriching the set of predicted preferences with the fewest possible number of queries. 
Numerical tests are provided to evaluate the richness and the reliability of
the predictions made. 
\end{abstract}

\section{Introduction}\label{sec:intro}
Preference elicitation or preference learning is an important step in setting up a recommender system for a Decision-Maker (DM). It usually consists in querying the DM, e.g. by asking her to assign alternatives to ordered categories.
By calling the learning procedure ``cautious'', we mean a procedure that complies with two principles that we now describe.

First, the sophistication of the learned multiattribute decision model should be adapted to fit the level of complexity of the stated preferences, hence the choice of a multiattribute utility function $f$ general enough to represent any order $\succ$ of preference, i.e., for any strict weak ordering $\succ$ on a set $\AlternativeSet$ of alternatives, there exists $f$ such that, for any pair $\{A,B\}\!\subseteq\!\AlternativeSet$, $f(A)\!>\!f(B)$ iff $A\!\succ\!B$. In particular, the multi-attribute model we use is able to model positive or negative interactions between attributes \cite{grabisch2008review}.
    
    Second, the predicted pairwise preferences should not depend on the partly arbitrary choice of precise numerical values for the parameters of the model but solely on the stated preferences, hence the design of an \emph{ordinal} learning procedure that maintains an isomorphism between the collected preferential data and the learned model (in the same spirit as ordinal measurement for problem solving \cite{bartee1971problem}) by using a polyhedron of possible values for the parameters, reflecting the uncertainty about them.
%==
As a consequence of this latter principle, when predicting an unknown pairwise preference between two alternatives $A$ and $B$, apart from the predictions ``$A$ is preferred to $B$'' and ``$B$ is preferred to $A$'', it is possible that the model does not make a prediction due to a lack of sufficiently rich preferential data (the absence of prediction is preferred to a bad prediction, although a compromise  must obviously be made between the reliability of the prediction and the predictive power of the learned model).
%As a consequence of this latter principle, when predicting an unknown pairwise preference between two alternatives $a_1$ and $a_2$, apart from the predictions ``$a_1$ is preferred to $a_2$'' and ``$a_2$ is preferred to $a_1$'', it is possible that the model does not make a prediction due to a lack of sufficiently rich preferential data (the absence of prediction will be preferred to a bad prediction, although a compromise  must obviously be made between the reliability of the prediction and the predictive power of the learned model).
%===
\paragraph{Problem.} We consider a multiattribute preference elicitation problem, where the attributes are assumed to be binary. Most elicitation procedures make an assumption of a numerical model, defined \emph{a priori}, underlying the DM's preferences. The originality of our approach consists in allowing the model to be revised during the elicitation process, by modifying the parameters space. The set of model parameters is thus defined \emph{a posteriori} from the preference statements. %We clarify this aspect in the following paragraph. 

\paragraph{A sparse model.} Following \citeauthor{Fishburn1996BinaryIA} \shortcite{Fishburn1996BinaryIA}, we consider an underlying numerical model $f$ where the value $f(S)$ of a set $S$ of attributes is an additive combination of parameters, one per subset $A$ of $S$: $f(S)\!=\!\sum_{A \subseteq S:A\in\theta} u_A$. While this model is general enough to model any strict weak ordering on the subsets of attributes, it is inherently intractable as there is a combinatorial set of parameters $u_A$. To keep a tractable set, %at each step of the elicitation process, 
similarly to the $k$-additive variant of this model (which only considers parameters $u_A$ for $|A|\!\le\!k$; \citeauthor{Fishburn1996BinaryIA} took $k\!=\!2$), we only consider a restricted family $\theta$ of subsets $A$.
We explore different strategies to design $\theta$ through the elicitation process. 
Our goal is to keep it minimal (in a formally defined sense), and yet general enough to fit the training set of pairwise preferences.

\paragraph{Cautious learning.} For each pair of alternatives, according to the collected preferential information, our learned model makes a \emph{cautious} prediction: it could either claim which alternative is preferred, or state that the collected information is not sufficient to conclude. In a nutshell, we only make predictions that are consistent with all the simplest models (following Occam's razor principle) able to explain the stated preferences.
The aim is to maximize the ratio of the number of correct preference predictions over the total number of predictions, while maintaining enough inference power.
%\paragraph{Elicitation procedure.} Since the preferences cannot be all known beforehand, we may need to build an interactive approach to query new preferences from the DM to further specify the model and thus make more predictions. This process called \emph{Incremental Elicitation} or \emph{Active Learning} entails the additional challenge of minimizing the number of solicitations of the DM and, more generally, the cognitive efforts that she will have to make. 

\paragraph{Organization of the paper.} After giving a brief review of the related work in Section~\ref{sec:related}, we present the $\theta$-additive utility model in Section~\ref{sec:sparse}, as well as the ordinal dominance relation that is inferred if the parameters are only partially specified. In Section~\ref{sec:minimal}, we show how to compute the simplest model compatible the collected preferences. Finally, Section~\ref{sec:tests} is devoted to numerical tests on synthetic preference data.
\section{Related work}\label{sec:related}
Preference elicitation, which is part of the broader framework of \emph{preference learning} (see e.g. \citeauthor{furnkranz2003pairwise} \shortcite{furnkranz2003pairwise}), has been studied for a long time in AI, as a preliminary step in any automation of a recommendation task. 

We focus here on the elicitation of the parameters of a multiattribute utility function taking into account interactions between attributes (more precisely, learning a partial specification of these parameters yielding a dominance relation between alternatives). In contrast with the setting of active learning which has been widely studied for preference elicitation (see e.g. \citeauthor{guo2010multiattribute} \shortcite{guo2010multiattribute}), we do not assume interactions with a DM but only the knowledge of a ``static'' training set of examples of pairwise preferences.
%precisely on the elicitation task from static examples, this category of tasks, unlike the incremental elicitation tasks, does not assume the possibility to interact with the user and therefore must be able to operate independently of it, knowing only a part of his preferences. \\
In this passive learning setting, many classification-based approaches has been proposed, going from perceptrons \cite{abs-1711-07875} to Gaussian processes \cite{chu05preference} or Support Vector Machines (SVM) \cite{domshlak2005unstructuring}. These approaches have in common that they consider, as a training set, a set of triples $(A,B,c)$, where $A$ and $B$ are two alternatives and $c\!=\!1$ if $A\!\succ\!B$, and $c\!=\!0$ otherwise.

A well-known multiattribute utility model that takes into account interactions between attributes, and closely related to the decision model we study in this paper, is the Choquet integral. One of the most recent work about the elicitation of the parameters of a Choquet-related aggregation function integral is that of \citeauthor{Bresson2020Learning2H} \shortcite{Bresson2020Learning2H}, in which in particular a perceptron approach is integrated into the learning process of a 2-additive hierarchical Choquet integral \cite{Bresson2020Learning2H}.
For a broad literature review about learning the parameters of a Choquet integral, the reader may refer to the article by \citeauthor{grabisch2008review} \shortcite{grabisch2008review}.
%a neural representation has been generalized to a non-linear case by replacing the classic linear output by a Choquet integral. 
Let us mention in particular the work by \citeauthor{marichal2000determination} \shortcite{marichal2000determination}, that use a polyhedron to characterize the set of parameters that are compatible with a training set of examples. The idea of defining a polyhedron of uncertainty on the parameters of a utility function goes back at least to the work of \citeauthor{charnetski1978multiple} \shortcite{charnetski1978multiple}. %who made an interesting use of the polyhedron of compatible utilities.
Their model state that $A\!\succ\!B$ if the proportion of parameters that give a better value for $A$ than for $B$ among those that are compatible with the stated preferences is greater than the proportion of parameters that give a better value for $B$ than for $A$.
This principle was also adapted to the case of a Choquet integral by \citeauthor{Angilella2015}~\shortcite{Angilella2015}. %, and a variant of it was proposed by \citeauthor{lahdelma1998smaa}~\shortcite{lahdelma1998smaa}) in order to deal with the impression in the decision-maker examples by computing a confiance factor on them.
In the sequel, 
we will use a similar polyhedron.% to characterize the set of parameters compatible with the collected preferences.

%More recently, for eliciting the utilities of attributes in an additive utility function, \citeauthor{saure2019ellipsoidal} \cite{saure2019ellipsoidal} proposed an approach that updates an ellipsoidal credibility region computed from a multivariate normal distribution over the parameters space, and showed how to use mixed integer programming to determine queries that are likely to reduce the volume of the credibility region.

%More recently, \citeauthor{zintgraf2018} \cite{zintgraf2018} extended this elicitation procedure to other query types (like asking the user to rank or cluster alternatives), and \citeauthor{astudillo2020multi} \cite{astudillo2020multi} explicitly took into account the uncertainty in utility estimation by proposing a \emph{set} of alternatives that seem appropriate for a range of plausible utility functions, from which the user makes a final selection. 

The two works probably closest to our proposal are those of \citeauthor{domshlak2005unstructuring} \shortcite{domshlak2005unstructuring} and \citeauthor{bigot2012using} \shortcite{bigot2012using}. For binary attributes, \citeauthor{domshlak2005unstructuring} consider a multiattribute utility function that is a sum of $4^n$ subutilities over subsets of attribute values and develop an efficient SVM approach to reveal this utility function, by relying on a kernel method. \citeauthor{bigot2012using} study the use of generalised additively independent decompositions of utility functions \cite{fishburn1970utility,gonzales:hal-01492604}. They give a polynomial PAC-learner when a constant bound is known on the function's degree, where the degree is the size of the greatest subset of attributes in the decomposition. Yet, both works do not fit the ``cautious learning'' framework we consider here.
\section{Our Cautious Learning Setting}\label{sec:sparse}
\subsection{Multiattribute Decision Problem}
In this paper, we tackle a multiattribute decision problem where alternatives are expressed in the form of a vector of binary attributes. 
Let $\mathcal{F}\!=\!\{a_1, a_2, \ldots , a_n\}$ be $n$ binary attributes and $\AlternativeSet\!\subseteq\!\{0,1\}^n$ be the set of alternatives defined on  $\mathcal{F}$. 
By abuse of notation, for $a_i\!\in\! \mathcal{F}$ and $A\!\in\!\AlternativeSet$, we will write $a_i\!\in\!A$ if the $i^{th}$ component of the vector characterizing $A$ is 1.
Moreover, for a subset $S\!\subseteq\! \mathcal{F}$ of attributes, we will write $S\!\subseteq\!A$ if $a\!\in\!A$ for all $a\!\in\!S$. 
For instance, if $A$ corresponds to $(1,1,1,0)$, then $\{a_1,a_3\}\!\subseteq\!A$.

%We assume that the DM has preferences, in the form of a weak order, over $\AlternativeSet$. 
%For $A,B\!\in\!\AlternativeSet$, we write $A\!\succeq\!B$ when she (weakly) prefers $A$ to $B$, $A\!\sim\!B$ when $(A\!\succeq\!B)$ and $(B\!\succeq\!A)$ (indifference), and  $A\!\succ\!B$ when $(A\!\succeq\!B)$ and $\lnot(B\!\succeq\!A)$ (strict preference).

We assume that the DM has preferences in the form of a strict weak order over $\AlternativeSet$. 
For $A,B\!\in\!\AlternativeSet$, we write $A\!\succ\!B$ when $A$ is strictly preferred to $B$, and $A\!\sim\!B$ when neither $A\!\succ\!B$ nor $B\!\succ\!A$ (incomparability).

The aim of preference elicitation is to predict strict pairwise preferences from a training set of examples.%The DM wishes to determine the ``best'' or a sufficiently good alternative on $ \AlternativeSet$.
\subsection{The $\theta$-additive Model} 
%==
\paragraph{Cardinal models and additive functions.} As the DM's preferences over $\AlternativeSet$ are modeled as a strict weak order, there exists a real-valued function $f$ such that $\forall A,B\!\in\!\AlternativeSet, f(A)\!>\!f(B)\!\Leftrightarrow\!A\!\succ\!B$. Many models assume that $f$ can be represented in a compact way using some sort of additive property. 

%
%\paragraph{Cardinal models and additive functions.} As the preferences of the DM over alternatives are modeled as a weak order, there exists a cardinal function $f$ such that $\forall A,B\!\in\!\AlternativeSet,\!f(A)\!\ge\!f(B)\!\Leftrightarrow\!A\!\succeq\!B$. Many models assume that this function $f$ can be represented in a more compact way using some sort of additive property. 
%===
One of the simplest and most used cardinal models for preference modelling in multiattribute utility theory is the 1-additive model \cite{keeney1993decisions}. This model makes the strong assumption that we can find a utility $u(a)\!\in\!\mathbb{R}$ for each attribute $a\!\in\!\mathcal{F}$ such that for all $A\!\in\!\mathcal{A}$, $f(A)\!=\!\sum_{a \in A } u(a)$. 
This assumption is strong because it implies that there is no interaction between the attributes.  
A weaker assumption is that of \emph{$k$-additivity} where we suppose the existence of a parameter $u(S)\!\in\!\mathbb{R}$ for each $S\!\in\![\mathcal{F}]^k$, where $[\mathcal{F}]^k\!=\!\{S\!\subseteq\mathcal{F}\!:\!1\!\le\!|S|\!\le\!k\}$. 
Hence, in the $k$-additive model, for all $A\!\in\!\AlternativeSet$, $f(A)\!=\!\sum_{S \in [\mathcal{F}]^k} I_A(S)u_S$, where $I_A(S)\!=\!1$ if $S \!\subseteq\!A$ and 0 otherwise, and $u_S$ is an abbreviation for $u(S)$. 
For example, the $2$-additive model makes it possible to account for binary interactions (positive or negative). 
The $n$-additive model is general enough to represent \emph{any} strict weak order on $\AlternativeSet$ because it can represent any real-valued set function $f\!:\!2^{\mathcal{F}}\!\rightarrow\!\mathbb{R}$ \cite{grabisch2000equivalent}, provided that $f(\emptyset)\!=\!0$.  
However, it requires to specify $2^n\!-\!1$ parameters. 
We therefore restrict our attention to additive models requiring fewer parameters.

\paragraph{The $\theta$-additive model.} 
In this paper, we consider a more flexible model which we call the $\theta$-additive model. 
Given a set $\theta\!\subseteq\! 2^{\mathcal{F}}$, and a set function
$u\!:\!\theta\!\rightarrow\!\mathbb{R}$, this model assumes that $f$ is of the form $f(A)\!=\!\sum_{S \in \theta} I_A(S) u_S$, where $u_S$ stands again for $u(S)$. 
In this case, we may also use the notation $f_{\theta,u}(A)$ instead of $f(A)$.  
Hence, the 1-additive model is the special case in which $\theta$ is $[\mathcal{F}]^1$, and the  $k$-additive model is the special case in which $\theta$ is $[\mathcal{F}]^k$. 

\begin{example}\label{ex : intro theta model}
Let $\mathcal{F}\!=\! \{a_1,a_2,a_3,a_4\}$ be a set of 4 attributes, $\AlternativeSet\!=\! \{0,1\}^{4}$ and the preferences of the DM be the strict weak order $\succ$ defined by:
\begin{align*}
     & & (0,1,1,1) & \succ\!\!\!  & (1,0,1,1)& \succ\!\!\! & (1,1,0,1)& \succ\!\!\! & (0,0,1,1) \\ & \succ\!\!\! & (0,1,0,1) & \succ\!\!\! & (0,1,1,0)& \succ\!\!\! & (1,0,0,1) & \succ\!\!\! &  (1,0,1,0) \\ & \succ\!\!\! & (1,1,0,0) & \succ\!\!\! & (0,0,0,1) & \succ\!\!\! & (0,0,1,0) & \succ\!\!\! & (0,1,0,0) \\ & \succ\!\!\! & (1,0,0,0) & \succ\!\!\! & \mathbf{(1,1,1,1) } & \sim\!\!\! & (0,0,0,0) & \succ\!\!\! & \mathbf{(1,1,1,0)}.
\end{align*}
These preferences can be explained by a clear negative interaction when attributes $a_1$, $a_2$, and $a_3$ are chosen together (vectors in bold). 
Interestingly, instead of using a 3-additive model, which would require the definition of 14 parameters, one can use the $\theta$-additive model with $\theta = \{\{a_1\},\{a_2\},\{a_3\},\{a_4\},\{a_1,a_2,a_3\}\}$ and $ u_{\{a_1\}} = 1$, $u_{\{a_2\}} = 2$, $u_{\{a_3\}} = 3$, $u_{\{a_4\}} = 4$, $u_{\{a_1,a_2,a_3\}} = -10$. 
\end{example}

\subsection{The Ordinal Dominance Relation}
We assume that we only have access to a partial set $R$ of strict pairwise preferences provided by the DM. This set may contain only a few comparisons. Our aim is to use these comparisons (observed preferences) in order to infer other strict pairwise preferences on the set of alternatives.
We formalize $R$ as a set of pairs $(A, B)\!\in\!\AlternativeSet^2$ such that $A\!\succ\!B$. %We denote by $\mathcal{R}$, the set of possible sets $R$.  
% in an incremental way (the size of $R$ will increase during the elicitation procedure). By abuse of notation, 

Moreover, given $\theta$, $U_{R}^\theta$ denotes the set of utility functions on $\theta$ that are compatible with the preferences observed in $R$:   
$$
U_{R}^{\theta} = \{u:\theta \rightarrow \mathbb{R} | \forall (A,B) \in R, f_{\theta,u}(A) > f_{\theta,u}(B)\}.
$$
Note that, for a given $\theta$, this set $U_{R}^{\theta}$ can be empty or composed of an infinity of possible utility functions on $\theta$. 
Notably, if this set is empty then the preferences of the DM cannot be represented by a $\theta$-additive function.

Viewing a $\theta$-additive function as a vector whose dimensions are the subsets $S$ in $\theta$, the set $U_{R}^{\theta}$ corresponds to the polyhedron defined by the following linear constraints in the $|\theta|$-dimensional parameter space (where each parameter $u_S$ corresponds to a dimension)\footnote{The right hand side of the constraint is here set to 1, but it could be set to any strictly positive constant as utilities $u_S$ are always compatible with $R$ to within a multiplicative factor.}:
\begin{equation}
\tag{P1}
\begin{split}
    \forall (A,B) \in R, \sum_{S \in \theta} I_A(S) u_S - \sum_{S \in \theta} I_B(S) u_S \ge 1. 
\end{split}
\end{equation}

For a given $\theta$, checking whether or not the preferences of the DM can be represented by a $\theta$-additive function can be evaluated in polynomial time by testing the consistency of the constraints in P1 (e.g., using a linear programming solver). 

We denote by $\Theta_{R}$ the set $\{\theta | U_{R}^\theta \neq \emptyset \}$, i.e., the $\theta$'s such that the preferences in $R$ are consistent with a $\theta$-additive function.
\begin{example}
Coming back to Example~\ref{ex : intro theta model}, setting $\theta = \{\{a_1\},\{a_2\},\{a_3\},\{a_4\}\}$ yields $U_{R}^\theta=\emptyset$. In contrast, setting $\theta_1 = \{\{a_1\},\{a_2\},\{a_3\},\{a_4\},\{a_1,a_2,a_3\}\}$ yields $U_{R}^{\theta_1}\neq \emptyset$. In this example, it can be shown\footnote{It has been computer tested by brute force enumeration.} that $\Theta_{R} = \{\theta : \theta_1 \subseteq \theta \}$.
\end{example}

As shown in the previous example, there may be several $\theta$ in $\Theta_{R}$. 
Moreover, for $\theta \in \Theta_{R}$, if $U_{R}^\theta$ is compounded of several compatible utility functions, then these utility functions may lead to quite different inferred preferences. 
%Note that   assuming that $R$ does not contain all the pairs of $\AlternativeSet$, there may be several compatible utility functions with $\theta \in \Theta_{R}$, which can give place to different comparisons between alternatives which do not belong to $R$. 

\begin{example}\label{ex : intro theta model 2}
Let $\mathcal{F} = \{a_1,a_2,a_3,a_4\}$. Let us assume that, contrary to Example~\ref{ex : intro theta model}, we now only observe preferences on the singletons $\{a_1\},\{a_2\},\{a_3\},\{a_4\}$: 
\begin{align*}
    R =& \{ ((0, 0, 0, 1), (0, 0, 1, 0)) , ((0, 0, 1, 0),\\
            &(0, 1, 0, 0)), ((0, 1, 0, 0), (1, 0, 0, 0)) \}
\end{align*}
The two additive functions $u$ and $u'$ defined by $u(\{a_1\})\!=\!1$, $u(\{a_2\})\!=\!2$, $u(\{a_3\})\!=\!3$, $u(\{a_4\})\!=\!5$ and $u'(\{a_1\})\!=\!1$, $u'(\{a_2\})\!=\!3$, $u'(\{a_3\})\!=\!4$, $u'(\{a_4\})\!=\!5$ are both in $U_{R}^{\theta}$, but we infer $(1,0,0,1) \succ (0, 1, 1, 0)$ from $u$ while we infer $(0,1,1,0) \succ (1, 0, 0, 1)$ from $u'$. %\MO{avons nous besoin de noter $u_S$ au lieu $u(S)$? ca fait un peu de confusion quand on ecrit après $u_1$, $u_2$ pour differencier les fonctions d'utilite. c'est vrai que $u_1$ est different de $u_\{1\}$, mais... }% the depending on whether we choose one or the other we have $(1,0,0,1) \succ (0, 1, 1, 0)$ or the converse.
\end{example}

This example shows that, given $R$, choosing a specific function $u \in U_{R}^{\theta}$ can lead to infer  preferences on the rest of $R$ that are only related to this arbitrary choice and not from the observed preferences~\cite{bartee1971problem}. % arbitrary choice of a precise function $u$ in $U_{R}^{\theta}$.
%preference information which is not necessarily implied by the observed preferences. 
As we will present in next sections, our aim is to infer preferences for pairs which do not belong to $R$ in a reliable way. In this purpose, we turn to an ordinal model based on the observed preferences which are in $R$. 

Fishburn and Lavalle \shortcite{Fishburn1996BinaryIA} showed how one can obtain an ordinal dominance relation from an underlying partially specified 2-additive numerical model. We now explain how the idea can be extended to an underlying $\theta$-additive model.

For a given $\theta\!\in\!\Theta_{R}$, the ordinal dominance relation is denoted by $\succ_{\theta}^{R}$, and is independent from the choice of a specific $u\!\in\! U_{R}^\theta$. 
This binary relation is defined, for each pair $A,B$ in $\AlternativeSet$, by: 
$$
A \succ_{\theta}^R B \Leftrightarrow \forall u \in U_{R}^\theta, f_{\theta, u}(A) > f_{\theta, u}(B).
$$
%Notations $\succ_{\theta}^R$ and $\sim_{\theta}^R$ are then defined as the asymmetric and symmetric parts of $\succeq_{\theta}^R$. 
Naturally, $(A,B)\!\in\!R\!\Rightarrow\!A\!\succ_{\theta}^R\!B$.
Nevertheless, note that binary relation $\succ_{\theta}^R$ is obviously partial, and we define the incomparability relation $\inc_\theta^R$ as: 
\begin{align*}
A \inc_\theta^R B &\Leftrightarrow 
\exists u , u' \in U_{R}^\theta, \\ 
&(f_{\theta, u}(A) \ge f_{\theta, u}(B) \and f_{\theta, u'}(B) \ge f_{\theta, u'}(A)).
\end{align*}
%When clear from the context, we omit $R$ and note the relations by $\succeq_\theta$ and $\inc_\theta$. \\
For any pair $A,B$ of subsets, one can test if $A\! \succ^R_\theta\!B$ in polynomial time, by considering the linear program where the objective function $\sum_{S \in \theta} u_S I_B(S)\!-\!\sum_{S \in \theta} u_S I_A(S)$ is maximized under constraints P1 (that characterize the set $U_\theta^R$ of compatible utility functions). The dominance $A\!\succ^R_\theta\!B$ holds iff the optimal value is strictly negative.% linear constraints P1 can be used to test if $A\! \succ^R_\theta\!B$ in polynomial time, by  maximizing the objective function $\sum_{S \in \theta} u_S I_B(S)\!-\!\sum_{S \in \theta} u_S I_A(S)$ for $u\!\in\!U_\theta^R$. We have $A \succ^R_\theta B$ iff the optimal value is strictly negative.
%the following objective function (where $u_S$ are the variables) known as the Pairwise Regret $PR(A,B,u)$:
%\begin{equation}
%    \begin{split}
%        PR(A,B,u) &= \sum_{S \in \theta} u_S I_B(S) - \sum_{S \in \theta} u_S I_A(S). \\
%    \end{split}
%\end{equation}
%Solving this linear program gives us the Pairwise Maximum Regret $PMR(A,B) \!=\!\max \{PR(A,B)\!:\!u\!\in\!U_\theta^R\}$. We have $A \succ_\theta^R B$ iff $PMR(A,B) < 0$.
%% PREUVE SUR LE CHANGEMENT DE R %% 

If $A\!\succ_\theta^R\!B$ then one can predict, based on $R$ and for a $\theta$-additive model, that $A$ is strictly preferred to $B$. If $A\!\inc_\theta^R\!B$ then no prediction is made.

\subsection{Sensitivity of the Ordinal Dominance Relation to Changes in $R$ or $\Theta$}

We now explore how the relation $\succ_\theta^R$ is modified when some new pairwise comparisons are added to $R$, or removed.
Interestingly, adding new pairwise comparisons to $R$ can only enrich binary relation $\succ_\theta^R$, provided the preferences remain representable by a $\theta$-additive function. Conversely, preference  $A \succ_\theta^R B$ cannot be reversed by removing pairwise comparisons from $R$. More formally:
%relation $\succ_\theta$ which depends on $R$, remain valid on all the supersets of $R$. Moreover, it cannot be contradicted in the subsets of $R$; thus, the preferences inferred by $\succ_{\theta}^{R'}$ will remain valid in $\succ_{\theta}^{R}$ provided that $U_\theta^R$ is not empty.
\begin{restatable}{proposition}{propChangeTheta} 
Given a set $R$ of strict pairwise comparisons, and $\theta\!\in\!\Theta_R$, if $R'\!\subseteq\!R$, then we have: (i) $\theta \!\in\!\Theta_{R'}$; (ii) $A \!\succ_{\theta}^{R'}
\!B\!\Rightarrow\!A\!\succ_{\theta}^{R}\!B$; (iii) $A\!\succ_{\theta}^{R}\!B\!\Rightarrow\! \neg (B \succ_{\theta}^{R'} A)$.
\end{restatable}

We now study how the relation $\succ_\theta^R$ is modified when $\theta$ is restricted or extended. If $\theta$ is restricted, then the relation $\succ_\theta^R$ can only be enriched. Conversely, if $\theta$ is extended, then a preference $A \succ_\theta^R B$ cannot be reversed after the extension. %More formally:

\begin{restatable}{proposition}{propChangeThetaBis}\label{propChangeThetaBis}
For $\theta,\theta'\!\in\!\Theta_R$, if  $\theta'\!\subseteq\!\theta$, then we have: (i) $A\!\succ_{\theta}^R\!B\!\Rightarrow\!A \!\succ_{\theta'}^R\!B$; (ii) $A \!\inc_{\theta'}^R\!B\!\Rightarrow\!A\! \inc_{\theta}^R\!B$; (iii) $A\! \succ_{\theta'}^R\!B\!\Rightarrow\!\neg (B \!\succ_{\theta}^R\!A)$.
%\begin{enumerate}
%\item[(i)] $A \succ_{\theta}^R B \Rightarrow  A \succ_{\theta'}^R B$;
%\item[(ii)] $A \inc_{\theta'}^R B \Rightarrow A \inc_{\theta}^R B$;
%\item[(iii)] $A \succ_{\theta'}^R B \Rightarrow \neg (B \succ_{\theta}^R A)$.
%\end{enumerate}
\end{restatable}

Note that many different $\theta$-additive models may be compatible with the collected preferences in $R$. In particular, if a $\theta'$-additive model is compatible with $R$, then any $\theta$-additive model such that $\theta$ extends $\theta'$ is also compatible with $R$. A natural way to decide which $\theta$-additive models to consider is to follow the inclusion relationship on $\Theta_R$, by considering the sets $\theta$ that are minimal w.r.t. inclusion. For computational efficiency, we will use a refinement of the inclusion relationship, that we detail in the next section.

%$\theta \subsetseq \theta$ any $\theta'$-additive model
%and it may hence seem difficult to decide which $\theta$-additive model to use. In the next section, we explore how to reason over the possible $\theta$-additive models compatible with the observed preferences.

\section{The Minimal Compatible Models and The Unifying Model}
\label{sec:minimal}

Note that there always exists a $\theta$ able to represent $R$; at worst, we can put all the subsets of $\mathcal{F}$ in $\theta$. Our choice of a specific $\theta$ among the various ones that yield a $\theta$-additive model able to explain the collected preferences in $R$ is guided by two criteria, namely: %We compare two $\theta$ in a lexicographic way and we note this comparison $\sqsupseteq_{lex}$:
\begin{itemize}
    \item First, following the philosophical principle of parsimony that the simpler of two explanations is to be preferred (Occam's razor \cite{blumer1987occam}), we consider subsets $\theta$ that minimize the complexity of interactions between the attributes; to measure this complexity, we use the \emph{degree} of $f_{\theta,u}$, namely $\max\{|S|:S\in \theta\}$ (i.e., the greatest cardinality of a subset of interacting attributes).
    \item Second, if two different $\theta$ have the same \emph{degree}, we prefer the one having \emph{sparsest} representation \cite{zhang2015survey}, i.e., the one which minimizes $|\theta|$ (which corresponds to the number of non-zero parameters $u_S$).
\end{itemize}

This two criteria define a lexicographic binary relation on $\Theta_R$, refining $\subseteq$ and denoted by $\sqsubseteq_{lex}$.
We call $\theta\!\in\!\Theta_{R}$ which are minimal according to $\sqsubseteq_{lex}$, \textit{simplest} $\theta$ of $R$ and we denote by $\Theta_R^{\min}$ their set: $\Theta_R^{\min}\!=\!\{\theta \in \Theta_{R}| \nexists \theta', \theta' \sqsubseteq_{lex} \theta\}$.

Note that sometimes the simplest model may contain more elements than another model which has a bigger degree: %as it is shown in the following example.
\begin{example}
Let $R= \{(1,1, 0,0) \succ (0,0,1,1) , (1,1,0,0) \succ (1, 0, 1, 0)\}$. It is easy to see that we can find a $\theta$ with one element containing a subset of cardinality 2 ($\theta\!=\!\{\{a_1,a_2\}\}$), however we will prefer having a $\theta$ consistent with a 1-additive model even if there are more elements in it : $\theta'=\{\{a_1\}, \{a_2\}\}$ or $\theta''=\{\{a_1\}, \{a_3\}\}$ or $\theta'''=\{\{a_2\}, \{a_3\}\}$.
\end{example}

%As it is shown in previous example, given $R$, one may have several $\theta$.

\subsection{Computation of $\Theta_R^{\min}$}

%In the following, we explain how we determine $\Theta_R^{\min}$ from $R$.
To compute the set $\Theta_R^{\min}$ from $R$, we perform an enumeration of all possible minimal $\theta$ sets by using Algorithm~\ref{alg : theta enumeration rec} called with $\Theta_R^{\min}\!=\! \theta\!=\!\emptyset$, and $\underline{\theta} \!=\!2^{\mathcal{F}}$.
%==
The parameters used by Algorithm~\ref{alg : theta enumeration rec} are the list $\Theta_R^{\min}$ under construction, a representative $\underline{\theta}$ of $\Theta_R^{\min}$ used to test whether $\theta$ is minimal w.r.t. $\sqsubseteq_{lex}$, the current $\theta$ under examination (i.e., whose membership to $\Theta_R^{\min}$ is being guessed) and the set $R$ of collected preferences. \\
%
%The parameters used by the Algorithm~\ref{alg : theta enumeration rec} are the list $\Theta_R^{\min}$ under construction, an element of $\Theta_R^{min}$ used as a representative to test whether a new found $\theta$ is dominated or not, $\theta$ that refers to the currently investigated $\theta$ and $R$ the set of preferences. \\
%===
%The pseudocode of this procedure is displayed in Algorithms~\ref{alg : theta enumeration rec} and \ref{alg : theta enumeration}. 

To perform this enumeration, we rely on : 
\begin{itemize}
    \item a depth first search strategy, where each node corresponds to a possible $\theta$, the root is initialized with $\theta\!=\!\emptyset$, and a node is expanded by investigating the possible sets $S$ that may break (i.e., invalidate) the certificate $I$ that $R$ is not compatible with a $\theta$-additive model %showing that $R$ cannot be represented by the $\theta$-additive model 
    (lines 8 to 11 in Algorithm~\ref{alg : theta enumeration rec}); we explain below how a certificate $I$ is defined and determined.
    \item a pruning strategy consisting in exploring only nodes who correspond to sets $\theta$ that are not dominated by the ones in $\Theta^{\min}_R$ w.r.t. $\sqsubseteq_{lex}$ (lines 2-3, 10 in Algorithm~\ref{alg : theta enumeration rec}).
\end{itemize}

\renewcommand{\algorithmicrequire}{\textbf{Input-Output:}}
\begin{algorithm}[h!]
\begin{algorithmic}[1]
  %\REQUIRE List $\Theta_R^{\min}$ under construction ; $\underline{\theta}$ a representative of $\Theta_R^{\min}$ ; $\theta$ under construction ; list $R$ of preferences.
  \IF{$R$ can be represented by a $\theta$-additive model}
            \IF{$\theta$ $\sqsubseteq_{lex}$ $\underline{\theta}$}
                \STATE $\Theta_R^{\min} \leftarrow \{\theta\}$; 
                \STATE $\underline{\theta} \leftarrow \theta$;
            \ELSE
                \STATE $\Theta_R^{\min} \leftarrow \Theta_R^{\min} \cup \{\theta\}$;
            \ENDIF
    \ELSE
            \STATE Find certificate $I$ and preference set $C$ by solving $\mathcal{D}_{\theta}$;\\
            \FOR{$S\!\in\!\{T \subset A\setminus B : (A,B) \in C \text{ or } (B,A) \in C\}$}
                \IF{$S$ breaks certificate $I$ and not $\underline{\theta} \sqsubseteq_{lex} \theta\cup\{S\}$}
                    \STATE BuildThetaMin$(\Theta_R^{\min}, \underline{\theta}, \theta \cup\{S\}, R)$;
                \ENDIF
            \ENDFOR
    \ENDIF
\end{algorithmic}
\caption{BuildThetaMin($\Theta_R^{\min}$,  $\underline{\theta}$, $\theta$, $R)$ }
\label{alg : theta enumeration rec}
\end{algorithm}
\renewcommand{\algorithmicrequire}{\textbf{Input:}}
%\begin{algorithm}[h!]
%    \begin{algorithmic}[1]
%        \REQUIRE List $R$ of preferences.
%        \ENSURE $\Theta_R^{\min}$ the set of minimal $\theta$ parameters such that $R$ can be represented by the $\theta$-additive model.
%        \STATE $\Theta_R^{\min}\leftarrow \emptyset$;
%        \STATE BuildThetaMinByEnumRec$(\Theta_R^{\min}, 2^{\mathcal{F}}, \emptyset, R)$;
%        \RETURN $\Theta_R^{\min}$;
%    \end{algorithmic}
%\caption{BuildThetaMinByEnum}
%\label{alg : theta enumeration}
%\end{algorithm}

%\subsection{How to update a $\theta$ from the collected preferences in $R$}

%To build a set $\theta$ we follow the general method presented in Algorithm~\ref{alg : theta generation}. 
%In what follows, we explain how each part of this Algorithm is performed. 
%\begin{algorithm}[h!]
%  \begin{algorithmic}[1]
%  \REQUIRE List of strict preferences in $R$; Initial $\theta$ composed of all singletons.
%  \ENSURE Set $\theta$ such that $R$ can be represented by the $\theta$-additive model. 
%  \WHILE{$R$ is not representable with the $\theta$-additive model}
%        \STATE Find a subset $C$ of preferences in $R$ that cannot be represented with the $\theta$-additive model, as well as a certificate $I$ for this incapacity;
%        \STATE Find a subset $S$ to add to $\theta$ to ``break'' the certificate $I$;          
%  \ENDWHILE
%  \RETURN $\theta$;
%\end{algorithmic}
%\caption{General method to incrementally build a set $\theta$ such that $R$ can be represented by the $\theta$-additive model.}
%\label{alg : theta generation}
%\end{algorithm}
\paragraph{Determining if $R$ can be represented by a $\theta$-additive model (line 1 of Algorithm \ref{alg : theta enumeration rec}).} Given a parameter set $\theta$, the following linear program $\mathcal{P}_{\theta}$, where there is one positive variable $e_{A,B}$ for each pair $(A,B)$ in $R$, and one free variable $u_S$ for each set $S$ in $\theta$, determines if the set $R$ of observed strict preferences can be represented by a $\theta$-additive model: 
\label{pl}
\begin{align*}
    (\mathcal{P}_{\theta}) \min_{e_{A,B},u_S} \sum_{(A,B) \in R } e_{A,B} &&\\
    \sum_{S \in \theta} (I_A(S) - I_B(S)) u_S &\ge 1 - e_{A,B}, &\forall (A,B) \in R \\
    e_{A,B} &\ge 0, &\forall (A,B) \in R 
\end{align*}
The preferences in $R$ can be represented by a $\theta$-additive model if the optimal value of $\mathcal{P}_{\theta}$ is 0. 
Indeed, in this case we can find values for variables $u_S$ that respect all the preferences in $R$ without the help of the additional slack variables $e_{A,B}$. 
%Note that this is always the case if we consider the linear program $\mathcal{P}_{2^{\mathcal{F}}}$, this special case $\theta = 2^{\mathcal{F}}$ will be of interest in this section. 

Program $\mathcal{P}_{\theta}$ is probably the most intuitive program to test if $R$ can be represented by the $\theta$-additive model. However, %for a reason that will soon be made clearer 
we will work instead on its dual $\mathcal{D}_{\theta}$:
\begin{align*}
    (\mathcal{D}_{\theta}) \max_{\lambda_{A,B}} \sum_{(A,B) \in R } \lambda_{A,B} &&\\
    \sum_{(A,B) \in R} (I_A(S)-I_B(S)) \lambda_{A,B} &=0,&\forall S \in \theta \\
    0\le \lambda_{A,B} &\le 1,& \forall (A,B) \in R
\end{align*}

If the optimal value of $\mathcal{D}_{\theta}$ is strictly positive, we must add at least another set to $\theta$ to represent the preferences in $R$. 

\paragraph{Finding a certificate (line 9 of Algorithm \ref{alg : theta enumeration rec}).} Let $I = (\lambda_{A,B}^* : (A,B) \in R)$ be an optimal solution to program $(\mathcal{D}_{\theta})$ such that $\sum_{(A,B) \in R } \lambda_{A,B}^* > 0$. 
Note that the values in $I$ make it possible to identify a set of preferences $C = \{(A,B) : \lambda_{A,B}^* > 0\}$ that cannot be represented by the current $\theta$-additive model, and that $I$ is in some sense a certificate for the incapacity to represent $C$ and thus $R$ (because $C\!\subseteq\!R$).  
In this case, one should add a set $T$ to $\theta$. 
This amounts to adding the constraint $\sum_{(A,B) \in R} (I_A(T)-I_B(T)) \lambda_{A,B} = 0$ to $\mathcal{D}_{\theta}$.
Importantly, note that this may only decrease the optimal value of $(\mathcal{D}_{\theta})$ if $\sum_{(A,B) \in R} (I_A(T)-I_B(T)) \lambda_{A,B}^* \neq 0$. 
Hence, the different candidates to add to $\theta$ will be precisely the sets $T$ that satisfy this condition.  
When adding such a set to $\theta$ we will informally say that we break $I$\footnote{This can be thought of as solving a separation problem, by providing an hyperplane separating $I$ from the polytope of $\mathcal{D}_{2^{\mathcal{F}}}$.}. 

\paragraph{Finding a set $S$ breaking $I$ (lines 10-14 of Algorithm \ref{alg : theta enumeration rec}).}
Note that a set $S$ breaking $I$ can always be found (even efficiently) as $R$ can be represented by any $\underline{\theta}$-additive model with $\{A,B : (A,B)\in R\}\subseteq \underline{\theta}$. Hence, a set $S$ breaking $I$ can always be found in $\{A,B : (A,B)\in R\}$. However, to keep $\theta$ ``simple'' we explore more systematically the sets that can break $I$ in order to find simple ones. In a nutshell, we enumerate all the sets in $\{S \subset A\setminus B : (A,B) \in C \text{ or } (B,A) \in C\}$.  
Indeed, each of these subsets may change the scores of sets appearing in $C$ and hence break the certificate $I$.  

%\paragraph{Termination of the process.}
%Note that the process iteratively solving $\mathcal{D}_{\theta}$ and adding sets to $\theta$ until the optimal value is 0 will converge as the optimal value of $\mathcal{D}_{2^{\mathcal{F}}}$ is 0.

\subsection{The Unifying Model}
\label{sec:unifying}

%From the different possible sets $\theta$ in $\Theta^{\min}_R$, we have to decide which $\theta$ we should use to infer new ordinal preferences. 
Instead of predicting $A\!\succ\!B$ if $A\!\succ_\theta^R\!B$ for all $\theta\!\in\!\Theta^{\min}_R$, we consider a single set $\theta$ ``synthesizing'' $\Theta^{\min}_R$ and infer preferences from it, because they are more easily explainable. %We showed that given some comparisons on $\AlternativeSet$ denoted $R$, we may have several simplest $\theta$-models compatible with $R$ and each of them may provide new comparisons on the rest of pairs in a different way. If we want to use them in order to complete the comparisons on $\AlternativeSet$, 
An intuitive idea consists of taking the union of all the simplest $\theta$. We call this model \textit{unifying model} and denote it by $\theta ^*_R$ :
$$\theta ^*_R = \cup_{\theta \in \Theta_R^{\min}} \theta.$$
Using the unifying model, we guarantee not to contradict the preferences that are compatible with all the $\theta$ in $\Theta_R^{\min}$.
\begin{proposition}
Let $R$ be the set of observed preferences on the elements of $\AlternativeSet$, let  $\Theta_R^{\min}$ be the set of simplest $\theta$-models compatible with $R$ and $\theta ^*_R = \cup_{\theta \in \Theta_R^{\min}} \theta$, then $\forall \theta \in \Theta_R^{\min}, \forall A, B \in \AlternativeSet$ 
$$ A \succ_{\theta^*_R}^R B \Rightarrow A \succ_\theta^R B.$$
\end{proposition}
%Hence, we can write: 
%$$\succ_{\theta^*}^R \subseteq \left(\bigcap_{\theta \in \Theta_R^{\min}} \succ_{\theta}^R\right)$$
%than those for which the strict preferences in $\succeq_{\theta^*}^R$ does not contradict any strict preference in $\succeq_{\theta}^R$ for $\theta \!\in\! \Theta_R^{\min}$

%As noted above, searching for the simplest and sparsest model will often lead to several possible models sharing the same level of complexity and sparsity. %with the same additivity and cardinality. \\
%In this part, we will study the impact of choosing one model rather than another and how we can choose a model that does not contradict any other model belonging to the simplest models; we will call such a model an unifying model. \\

Unfortunately the inverse is not true, i.e, it is possible that $A\!\succ_\theta^R\!B$ for $\theta\!\in\!\Theta_R^{\min}$ but not $A\!\succ_{\theta^*_R}^R\!B$. Example~\ref{ex:counter_unifying} in appendix illustrates this point.
\section{Numerical Tests}\label{sec:tests}
%==
Numerical tests were carried out on Google Colab (2 virtual CPU at 2.2GHz, 13GB RAM). The objective of these tests is twofold: 1) evaluating the \emph{accuracy rate} of the predictions, namely the number of correct pairwise preference predictions over the total number of predicted preferences, if the set $\theta$ is \emph{known} beforehand; %, where a prediction is considered as wrong if $A$ is preferred to $B$ but one predicts the opposite; 
%2) evaluating the same metric if the set $\theta$ is \emph{unknown} beforehand and learned with Algorithm~\ref{alg : theta enumeration rec}.%\emph{stability} of the predictions, i.e., the number of times a predicted preferences has been reversed after $t$ queries. 
%
%Numerical tests have been carried out on Google Colab (2 virtual CPU at 2.2GHz, 13GB RAM). The objective of these tests is twofold: 1) evaluating the \emph{accuracy rate} of the predictions, namely the number of correct pairwise preference predictions over the total number of predicted preferences, if the set $\theta$ is \emph{known} is beforehand; %, where a prediction is considered as wrong if $A$ is preferred to $B$ but one predicts the opposite; 
2) evaluating the same metric if the set $\theta$ is \emph{unknown} beforehand and learned with Algorithm~\ref{alg : theta enumeration rec}.%\emph{stability} of the predictions, i.e., the number of times a predicted preferences has been reversed after $t$ queries. 
%===

\subsection{The Tier List Framework}
We place ourselves in an elicitation context where each query consists in asking the DM to position an alternative in a tier list of ordered classes (i.e., the worst alternatives in category 1, the second worst alternatives in category 2, etc.). %\MO{J'ai changé un peu la suite et j'ai mis classe ou lieu de "rank" car rank pourra faire reference a un rangement plus qu'une classification ordinal... mais c'est a voir}
Formally, we assume that the user gives us access to a function $\TierFunction :\AlternativeSet \rightarrow \mathbb{N}$ that associates each alternative to a class in the tier list such that $\TierFunction(A) > \TierFunction(B) \Rightarrow A \succ B$. Note that $\TierFunction(A) = \TierFunction(B)$ does not mean here that $A$ and $B$ are indifferent, but that the user do not know how to compare them.%we assume that we do not know how to compare two alternatives belonging to the same class ($\TierFunction(A) = \TierFunction(B)$).% we don't conclude that $A \sim B$ so the ranks does not correspond to an ordinal utility function. 
%formally, this tier list can be seen as an ordinal scale where we associate each alternative with a rank such that each alternative is preferred to all the one of lower rank. 

Positioning one alternative in the tier list allows us to interactively collect numerous strict pairwise preference relations while keeping a low cognitive burden compared to asking for pairwise comparisons or for scores (one score per alternative).% Indeed, positioning one alternative in the tier list will often imply numerous pairwise preference relations.

%that is while going through interactions much less cognitively costly than asking for a score and more efficient than asking for pair comparisons. 
%Indeed, positioning one alternative in the tier list will often imply numerous pairwise preference relations.

\subsection{Synthethic Generation of a Tier List}
%==
This section details our simulation of the creation of a tier list from a $\theta$-additive function modeling the DM's preferences.
%
%This section details how we simulate the creation of a tier list compatible with a $\theta$-additive function modeling the preferences of a DM.
%===
%by a DMs  we generate an oracle that will simulate 
%a DM; the DM will interact by ranking the alternatives in a tier list; a sampled $\theta$-additive function will condition his preferences.} 
\subsubsection{Sampling  a $\theta$-additive Function $f_{\theta, u}$} %with $u \in U_\theta^R$}
For sampling a function $f_{\theta, u}$, we first sample a set $\theta$ and then sample parameters $u_S$ for $S\!\in\!\theta$. %a utility parameters $u$ in $U_\theta^R$ by sampling the different coefficients $w_S$ 
%\MO{il faut homogeniser les notations sur $w_S$ et $U_S$...}
%==
More precisely, the generation of $\theta$ is achieved as follows. First, $\theta$ is initialised as the set of singletons $\{a_1\},\{a_2\},\ldots,\{a_n\}$, then we add $\lfloor \alpha \times (2^{|{\mathcal{F}}|} - |\mathcal{F}|) \rfloor$ subsets of attributes, where the coefficient $\alpha\!\in\![0,1]$ makes it possible to control the model's complexity: for $\alpha\!=\!0$, only the singletons are in $\theta$, which yields the simple additive utility model, and for $\alpha\!=\!1$, all subsets of attributes are present, with yields the most general utility model. %proportion of subsets that are present in the model.
%
%More precisely, the generation of $\theta$ is achieved as follows. First, $\theta$ is initialised as the set of singletons $\{a_1\},\{a_2\},\ldots,\{a_n\}$, then we add $\lfloor \alpha \times (2^{|{\mathcal{F}}|} - |\mathcal{F}|) \rfloor$ subsets of attributes, where the coefficient $\alpha\!\in\![0,1]$ makes it possible to control the complexity of the model: for $\alpha\!=\!0$, only the singletons are present in $\theta$, which yields the simple additive utility model, and for $\alpha\!=\!1$, all subsets of attributes are present, with yields the most general utility model. %proportion of subsets that are present in the model.
%===
Each subset $S$ is sampled according to a parameter $p\!\in\!(0,1]$:
\begin{enumerate}[noitemsep]
    \item Initialize $S$ as a singleton by uniformly sampling in $\mathcal{F}$.
    \item Uniformly sample another attribute in $\mathcal{F}$ and add it to $S$.
    \item Exit this process if $
    S\!=\!\mathcal{F}$.
    \item Exit this process with a probability $p$ otherwise go to 2.
\end{enumerate}
The expected size of each $S$ we add can be approximated by:
\begin{equation*}
    \mathbb{E}[|S|] = 2 + (1-p-(1-p)^{n-1})/p.
\end{equation*}
%\OS{Revoir la formule car $|\mathcal{F}|\!\neq\!\infty$.}
%We reiterate this process over and over again until we have our $\alpha (2^{|{ \mathcal{F} }|} - |\mathcal{F}|)$ coefficients. \\
Table~\ref{tab:sizeS} gives some hint of the expected size of each $S$ according to $p$. Once $\theta$ is set, we sample the parameters $u_S$ for each $S\!\in\!\theta$ with a normal distribution $\mathcal{N}(0,\sigma)$. The sampling  of $f_{\theta,u}$ thus depends on three parameters $p$, $\alpha$ and $\sigma$. In the tests, $p$ varies in $[0.1, 0.9]$, $\alpha$ in $[0.1, 0.5]$, and we set $\sigma\!=\!100$.%, the others parameters will vary during the study in $[0.1, 0.9]$ and $[0.1, 0.5]$ respectively.

\begin{table}[hbtp]
    \centering
    \begin{tabular}{cccccc}
    \hline
        $p$ & 0.2 & 0.4 & 0.6 & 0.8 & 1 \\
    \hline
        $\mathbb{E}[|S|]$ & 3.95 & 3.18 & 2.62 & 2.25 & 2.00 \\
    \hline
    \end{tabular}
    \caption{\label{tab:sizeS}Expected size of subsets $S$ w.r.t. $p$.}
\end{table}

\begin{example}
\label{ex:tests}
If $n\!=\!4$, $p\!=\!0.3$, $\alpha\!= \!0.1$, then $\lfloor 0.1(2^5\!-\!5) \rfloor\!=\!2$ subsets $S$ are sampled in addition to the singletons. This may yield the parameter values given in Table~\ref{tab:ex_parameter_values}.
%are thus sampled $\lfloor 0.1(2^5 - 5) \rfloor = 2 $ interaction coefficients; this may yield :
\begin{table}[h!]
\centering
\begin{tabular}{|l|l||l|l|}
\hline
\textbf{Subset} & \textbf{Value} & \textbf{Subset} & \textbf{Value} \\ \hline
\{0\}           & 148.85    & \{4\}           & 191.00             \\ \hline
\{1\}           & 186.75    & \{1,3,4\}       & -26.80            \\ \hline
\{2\}           & 90.60     &  \{0,2\}         & 80.24           \\ \hline
\{3\}           & -86.12    & &            \\ \hline
\end{tabular}
\caption{\label{tab:ex_parameter_values}Example of parameter values.}
\end{table}   
\end{example}

\subsubsection{From $f_{\theta,u}$ to a Tier List} %Deducing the $\xi$ function}
The function $\TierFunction\!:\! \AlternativeSet\!\rightarrow\!\mathbb{N}$ that simulates the user assignment of alternatives into a tier list, called tier function hereafter, relies on a parameter $t$ representing the number of categories. The range of scores $f_{\theta,u}(A)\!=\!\sum_{S \in \theta} u_S I_A(S)$ 
of alternatives $A$ is partitioned into $t$ equally-sized intervals between the min score $f_0\!=\!\min_{A \in \AlternativeSet} f_{\theta,u}(A)$ and the max score $f_{t}\!=\!\max_{A \in \AlternativeSet} f_{\theta,u}(A)$.
%To define $\TierFunction$, we start by computing $f_{t} = \max_{A \in \AlternativeSet}(\sum_{S \in \theta} u_S I_A(S)) $ and $u_0 = \min_{A \in \AlternativeSet}(\sum_{S \in \theta} u_S I_A(S)) $. Then, we divide the interval $[u_0, u_{t}]$ into $t$ equally-sized intervals, $[u_{0}, u_1]$, $(u_1, u_2], \ldots, (u_{t-1}, u_{t}]$. 
The function $\TierFunction$ is then defined by: 
\begin{equation*}
    \TierFunction(A) = \min\{1\le k\le t : f_{\theta,u}(A) \leq u_k \}.
\end{equation*}
Put another way, we associate to each subset the interval where its utility lies.
%Note that if we take $t = |\AlternativeSet|$ the function gives us a full order on the alternatives \HG{Not true !} 
In general, the more categories we add, the less incomparabilities we will have (alternatives assigned to the same category), but the user will have to make more efforts to assign the alternatives to categories.

\begin{example}
Coming back to Example~\ref{ex:tests}, let $\AlternativeSet\!=\!\{0,1\}^n$. Then $\max_{A \in \mathcal{A}} f_{\theta, u}(A)\!=\!616.41$ and  $\min_{A \in \mathcal{A}} f_{\theta, u}(A)\!=\!-86.12$. Assume that one partitions into $t=3$ categories. The intervals are then
$[-86.12, 148.05],$ $(148.05,382.23]$ and $(382.23,616.41]$. Subset $\{1,2,3\}$ is then assigned to category $2$ because its utility $191.23$ belongs to $(148.05,382.23]$.
\end{example}
\subsection{Baseline Models}
In the following, the ordinal model studied in the paper is denoted by ORD. In this part, we will briefly introduce the baseline models to which ORD is compared.%we use in order to evaluate the performances of our approach.
\paragraph{Linear Programming Model (LPM).}
%==
As a first baseline model, we compare our approach with the model consisting in setting parameters $u_S$ at their optimal values for the linear program $\mathcal{P}_{\theta}$ of page~\pageref{pl}, and predicting that $A \!\succ\!B$ if $f_{\theta,u}(A)\!>\!f_{\theta,u}(B)$. In the experiments, if $\theta$ is known beforehand, only the constraints set of $\mathcal{P}_{\theta}$ grows, while if $\theta$ is unknown, both the variables and the constraints may change when $R$ grows.
%
%As a first baseline model, we compare our approach with the model consisting in setting parameters $u_S$ at their optimal values for the linear program $\mathcal{P}_{\theta}$ of page~\pageref{pl}, and predicting that $A \!\succ\!B$ if $f_{\theta,u}(A)\!>\!f_{\theta,u}(B)$. In the experiments, if $\theta$ is known beforehand, only the set of constraints of $\mathcal{P}_{\theta}$ grows, while if $\theta$ is unknown, both the set of variables and the set of constraints may change when $R$ grows.
%===
%
%linear program 
%a cardinal model. This model consists given a fixed $\theta \in U_\theta^R$ in a single $\theta$-additive function defined by it's coefficient $u_S$ for $S \in \theta$ that will be fit on the preferences of the user by minimizing the sum of the slack variables $e_{A,B}$ associated with the preferences-related constraints of the form $\sum_{S \in \theta} (I_A(S) - I_B(S) ) u_S \geq 1 - e_{A,B} $ for each $(A,B) \in R$. 
%Formerly, this model is the solution of the linear program $\mathcal{P}_{\theta}$ of page~\pageref{pl}.
%\HG{I don't get this baseline.} \MO{Je crois que l'on fixe un theta...} \OS{The presentation has now been changed!}
\paragraph{Support Vector Machine (SVM).} This baseline model is inspired by an approach proposed by \citeauthor{domshlak2005unstructuring} \shortcite{domshlak2005unstructuring}.
%==
An SVM approach is a supervised learning method for binary classification: each example in the dataset is labeled by 0 or 1; an SVM is learned from the dataset, from which labels are inferred for new examples. In our setting, each preference $A\!\succ\!B$ in $R$ yields two examples: a $(2m\!+\!1)$-dimensional vector $(v_A^\theta, v_B^\theta, 1)$ and another vector $(v_B^\theta, v_A^\theta, 0)$. That is, the third component of $(v_A^\theta, v_B^\theta, c)$ is $c\!=\!1$ if $A$ is preferred to $B$, and $c\!=\!0$ if it is not. Note that, when inferring labels (and thus predicting preferences), it may happen that $(v_A^\theta, v_B^\theta)$ and $(v_B^\theta, v_A^\theta)$ get the same label (0 or 1). In this case, no strict preference is predicted.

\subsection{Experiment with a Known $\theta$}% \HG{Find better name and use Experiment instead of Experience!}}
In the first experiment, we compared the two above baseline models with our ordinal model when the $\theta$ used to generate the tier function $\gamma$ is known beforehand.%with knowing the underlying $\theta$ used to generate the preferences.
\paragraph{Used metrics.}
To evaluate the accuracy of each model, we rely on the following measures:
\begin{itemize}[noitemsep]
    \item Correct answers (C): an inferred preference $A\!\succ\!B$ is said to be correct if $\TierFunction(A)\!>\!\TierFunction(B)$.
    \item  Wrong answers (W): an inferred preference $A\!\succ\!B$ is said to be wrong if $\TierFunction(A)\!<\!\TierFunction(B)$.
\end{itemize}
Given a model (ORD, LPM or SVM), a preference between $A$ and $B$ is \emph{inferred} if the preference is \emph{not} already present in $R$ and the model states that $A\!\succ\!B$ or $B\!\succ\!A$ (but not both).
We denote by $T$ the total number of inferred preferences. Note that $C\!+\!W \!\le\! T$ because it may happen that $\gamma(A)\!=\!\gamma(B)$. % among the inferred preferences some of them are neither correct nor wrong. \\
The Absolute Correct Rate (ACR) is defined from $T$ and $C$:
$$
ACR = C/T
$$
\paragraph{Experimental setting.}
The experiment was conducted with $|\mathcal{F}|\!=\!5$, $t\!=\!12$ $\sigma\!=\!100$, and two sets of parameters $(\alpha,p)$, namely $(\alpha,p)\!=\!(0.1,0.9)$ and $(\alpha,p)\!=\!(0.3,0.7)$. Roughly speaking, the former set of parameters generates tier functions with low interactions, while the latter generates tier functions with high interactions. %$p$ varying in $\{ 0.1, 0.3, 0.5, 0.7, 0.9 \}$, $\alpha$ varying in $\{0.1, 0.3, 0.5, 0.7\}$, and $t$ varying $\{3, 6, 9, 12, 15, 18\}$. %\MO{je crois qu'il y a une erreur dans l'ordre des 3 ensembles}.\\
For each couple $(\alpha, p)$, we sample three random tier functions and, for each one, we train each model with a budget of $25$ assignments to categories. %in each step we enrich the tier-list with a random subset, and then we use the set of preferences induced by the tierlist to train the 2 cardinal models : SVM and MS. \\
%To evaluate the models, we sample 5 times 10 subsets randomly and evaluate different metrics over each model's induced preferences on these subsets.
The test examples are generated as follows: we randomly sample 10 alternatives $A_1,\ldots,A_{10}$ in $\mathcal{A}$ and we consider all pairs $\{A_i,A_j\}$ for $i\!\neq\!j$. We count the number $T$ of inferred preferences for these pairs, and we evaluate the ACR. To smooth the results, they are averaged over 10 different tier functions, and 5 samples of ten alternatives for each of them. %presented in Tables~\ref{fig:Exp1Set1} and \ref{fig:Exp1Set2}, are averaged over 5 such samples.
\paragraph{Results and discussion.}
%The results are presented in Tables~\ref{fig:Exp1Set1} and \ref{fig:Exp1Set2}.
%\MO{il faut trouver une solution pour la figure qui sort de la colonne}
%\MO{il faut expliquer comment tu calcules wrong et correct, par rapport a quoi, ... ca fait bizarre d'avoir des corrects qui chutent pour SVM et MS...}
%However, these results also highlight that the choice between a cardinal or ordinal model is a trade-off between the number of predictions and their robustness.\MO{je suis d'accord mais dans ce cas, trouvons des indicateurs qui montre ca... combien de predictions on arrive a faire? il faut peut etre trouver une autre moyen de presenter pour pouvoir mettre ensemble correct, wrong and undetermined... des barres a trois couleurs ou a.a chose comme ca... ahh j'ai vu apres qu'il ya figure 2 qui n'est pas reference... Je ne sais pas si on a besoin de fig1 ET fig 2, a reflechir, peut etre les barres a 3 couleurs resoudront le probleme:))} \\
The results are presented in Figures~\ref{fig:Exp1Set1} and \ref{fig:Exp1Set2}, where the x-axis gives the size of the training set and the curves show the mean and 95\% confidence interval. The curves show how the average number of inferred preferences and the average ACR evolve with the size of the training set (from 1 to 25 assignments of alternatives to categories). In both figures, we see that the number of inferred preferences grows more slowly with ORD than with LPM and SVM, in accordance with the principle of cautious learning. However, the accuracy is better, as reflected by the curve of ACR for ORD that is consistently above the curves obtained for LPM and SVM. As one could expect, when the interactions are high (Figure~\ref{fig:Exp1Set2}), and thus the number of parameters $u_S$ is significant, a larger learning set is required to make it possible to infer numerous pairwise preferences with ORD. Note that, when the number of assignments available in the training set is low, the confidence interval for the curve of ACR for ORD is wide. This is related to the fact that few preferences are inferred and therefore a wrong prediction drastically change the ACR. However, after 15 assignments, the number of inferred preferences becomes higher, and the ACR for ORD outperforms the ACR for LPM and SVM. Comparing Figure~\ref{fig:Exp1Set1} and Figure~\ref{fig:Exp1Set2}, we can even see that, after 25 assignments, the difference in ACR is greater with high interactions than with low interactions. We ascribe this to the fact that the three models behave similarly with low numbers of parameters $u_S$ ($|\theta|$ not far from $n$) because the polyhedron of compatible utilities is small. We also notice in the two figures that the number of inferred preferences is always greater with LPM and SVM than with ORD. Put another way, ORD represents a different trade-off between the number of preferences that can be predicted and their accuracy.
%The following figures shows the evolution of the ACR and of the number of inferred preferences for our model and the three other baseline models during the process of elicitation. \\
%The difference between the two figures stands in the parameters $(p,t,\alpha)$ that were used to carry out the experience.
\begin{figure}[h]
    \centering
    \includegraphics[scale=0.55]{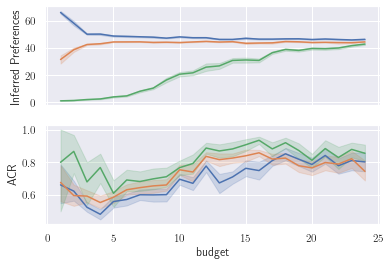}
    \caption{$T$ (top) and ACR (bottom) with ORD (green), LPM (blue) and SVM (orange) for a known $\theta$ and $(\alpha,p,t)\!=\!(0.1,0.9,12)$.} %The x-axis gives the size of the training set. The curves show the mean and 95\% confidence interval.}
    \label{fig:Exp1Set1}
\end{figure}
%In the first figure, $\alpha$ and $p$ are relatively small and we see that even if the ordinal model have a better accuracy, the gap between the models are small we ascribe this to the fact that cardinal and ordinal models behave similarly when the polyhedron of compatible utilities is small, which is the case when theta is small.
\begin{figure}[h]
    \centering
    \includegraphics[scale=0.55]{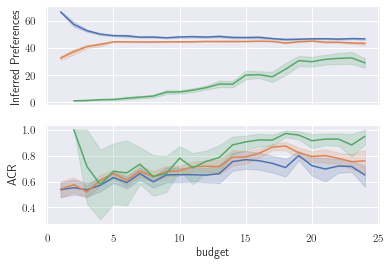}
    \caption{$T$ (top) and ACR (bottom) with ORD (green), LPM (blue) and SVM (orange) for a known $\theta$ and $(\alpha,p,t)\!=\!(0.3,0.7,12)$.} %The x-axis gives the size of the training set. The curves show the mean and 95\% confidence interval.}
    \label{fig:Exp1Set2}
\end{figure}
%In this second figure, we can see that the ordinal model still outperform the two other baseline models with a bigger gap.
%We also notice in the two figures that the inferred preferences are always greater with the two cardinal models, and so, both the cardinal and the ordinal models represent two different trade off between the number of preferences that could be inferred and their accuracy.

\subsection{Experiment with an Unknown $\theta$}%\HG{Find better name and use Experiment instead of Experience!}

In this section, we investigate the behavior of the models when $\theta$ is learned at the same time as parameters $u_S$ ($\forall S\!\in\!\theta$).% (for $S\!\in\!\theta$).%, and we compare it to LPM and SVM.

\paragraph{Experimental setting.}
The experimental setting is similar to the previous one, except that the number of categories in the tier lists is set to $t\!=\!9$. %with the notable exception that we assume that we do not know the underlying model $\theta$ of the user. Furthermore, the number of categories in the tier lists is set to $t\!=\!9$. %the only difference lies in the fact that in this one we assume that we do not know the underlying model $\theta$ of the user. \\
For all models (ORD, LPM and SVM), the set $\theta$ is updated after each assignment of an alternative to a category, by using Algorithm~\ref{alg : theta enumeration rec}. 

\paragraph{Results and discussion.}
The results are presented in Figures~\ref{fig:Exp2Set1} and \ref{fig:Exp2Set2}, with the same conventions as above. Similarly to the case of a known $\theta$, we see that model ORD outperforms models LPM and SVM in terms of accuracy. We notice small irregularities in the inferred preferences curve of ORD, due to the fact that ORD infers less preferences each time $\theta$ is updated because the polyhedron of compatible parameters expands when dimensions are added (corresponding to new subsets in $\theta$). Figure~\ref{fig:Exp3} shows the result of another experiment where the models are trained twice: once using the actual $\theta$ used to generate the synthetic preferences in $R$, and a second time using the $\theta$ obtained by computing a unifying model (see Section~\ref{sec:unifying}). %after enumerating the elements of $\Theta_{R}^{\min}$ with Algorithm~\ref{alg : theta enumeration rec}. 
Interestingly, both learning curves are close to each other, which tends to show that the learned $\theta$ is relevant.

\begin{figure}[!h]
    \centering
    \includegraphics[scale=0.55]{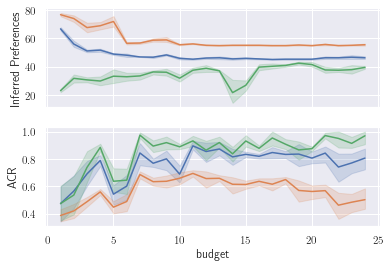}
   \caption{$T$ (top) and ACR (bottom) with ORD (green), LPM (blue) and SVM (orange) for an unknown $\theta$ and $(\alpha,p,t)\!=\!(0.1,0.9,9)$.}% The x-axis gives the size of the training set. The curves show the mean and 95\% confidence interval.}
    \label{fig:Exp2Set1}
\end{figure}

\begin{figure}[!h]
    \centering
    \includegraphics[scale=0.55]{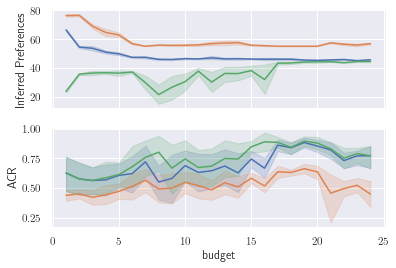}
    \caption{$T$ (top) and ACR (bottom) with ORD (green), LPM (blue) and SVM (orange) for an unknown $\theta$ and $(\alpha,p,t)\!=\!(0.3,0.7,9)$.}% The x-axis gives the size of the training set. The curves show the mean and 95\% confidence interval.}
    \label{fig:Exp2Set2}
\end{figure}

\begin{figure}[!h]
    \centering
    \includegraphics[scale=0.55]{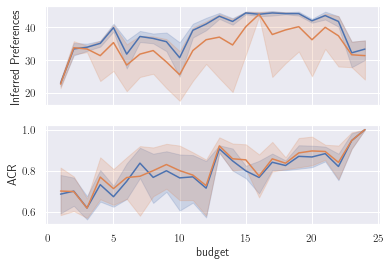}
    \caption{Number of inferred preferences and ACR with the real $\theta$ (in orange) and with $\theta^*_R$ (in blue), for  $(\alpha,p,t)\!=\!(0.3,0.7,9)$.}% The x-axis gives the size of the training set. The curves show the mean and 95\% confidence interval.}
    \label{fig:Exp3}
\end{figure}

\section{Conclusion}
We have presented here a ``cautious'' method for learning pairwise multiattribute preferences. The model we use is not restrictive, in the sense that any preference relation on the space of alternatives can be represented. The learning method achieves a trade-off between the number of predicted preferences and the accuracy of the predictions, by relying on an ordinal dominance relation between alternatives.
%Cardinal elicitation models tend when trying to quantify the alternatives to make presuppositions; these presuppositions concern the parameters of the utility function that we wish to learn and the choice of a particular function since several functions can be compatible with the same set of preferences.
%These presuppositions are arbitrary choices that our method avoids by producing preferences that are not contradicted by any function of any model belonging to the simplest models that represent preferences. \\

Several research directions are worth investigating, among which the adaptation of the approach to an active learning setting where one interactively determines a sequence of queries to minimize the cognitive burden for a DM, or the examination of other definitions of the set $\Theta_R^{\min}$ of simplest models compatible with $R$.

\section*{Acknowledgements}
We acknowledge a financial support from the project THEMIS ANR20-CE23-0018 of the French National Research Agency (ANR).

%the refinement of the definition of the unifying model $\theta^*_R$ in such a way that $\succ_{\theta^*_R}^R$ gets even closer to the preferences on which all the simplest models agree.
%to obtain a single set of apra
%$\theta$ of parameters that yields preferences that are closer to the intersection   take into account the fact that preferences that do not contradict any model do not necessarily have to belong to all models.

%Moreover, since our method is based on a polyhedron of compatible functions, it opens several perspectives on incremental elicitation by asking how to select the question that is most likely to constrain this polyhedron.\\
%Finally, the concept of unifying model still needs to be explored insofar as the unifying model used infers preferences that are part of the intersection of preferences produced by each model.
%There is still considerable room for improvement since preferences that do not contradict any model do not necessarily have to belong to all models.
\bibliographystyle{named}
\bibliography{ecml2022}

\newpage
\appendix
\section*{Appendix}

\propChangeTheta*
\begin{proof} 
$(i)$ If all the preferences in $R$ can be represented by a $\theta$-additive function, then so can the preferences in $R'$ as $R'$ is compounded of a subset of the preferences in $R$. 

$(ii)$ If the preferences in $R'$ imply that $A$ should be necessarily strictly preferred to $B$, then $R$ will imply the same condition as $R$ also contains the same preference constraints as in $R'$. 

$(iii)$ The contrapositive is proved as follows: $B\!\succ_\theta^{R'}\!A\!\Rightarrow\!B\succ_\theta^R\!A$ by $(ii)$, and $B\succ_\theta^R\!A\!\Rightarrow\!\neg(A\succ_\theta^R\!B)$ because strict preferences are asymmetrical. 
\end{proof}

\propChangeThetaBis*
\begin{proof}
$(i)$ is true because if $f_{\theta, u}(A) > f_{\theta, u}(B)$ for all $u \in U_{R}^\theta$, then we should also have $f_{\theta', u}(A) > f_{\theta', u}(B)$ for all $u \in U_{R}^{\theta'}$. Indeed, each element of $U_{R}^{\theta'}$ can be seen as a utility function in $U_{R}^{\theta}$ in which the parameters $u_S$ are set to 0 for $S\in \theta\setminus \theta'$.

$(ii)$ follows by a similar argument as for $(i)$. 

$(iii)$ The contrapositive is proved as follows: $B\!\succ_\theta^{R}\!A\!\Rightarrow\!B\succ_{\theta'}^R\!A$ by $(i)$, and $B\succ_{\theta'}^R\!A\!\Rightarrow\!\neg(A\succ_{\theta'}^R\!B)$ because strict preferences are asymmetrical.
%is obtained by considering the contraposition of $(i)$ and the asymmetry of $\succ_{\theta'}$. Lastly, $(ii)$ 
\end{proof} 

\begin{example}
\label{ex:counter_unifying}
 Let's take $\mathcal{A} = \{0, 1\}^4$ ($\mathcal{F} = \{a_1, a_2, a_3, a_4\}$) and observed preferences $R$ as in the following: 
$$
(1,1,1, 0) \succ (0,0,0,1) \succ \emptyset \succ (0,1,1,0).\\
$$

We have $\Theta_R^{\min}\!=\!\{ \theta_1, \theta_2 \}$ with $\theta_1\!=\!\{a_1,a_3,a_4\}$, $\theta_2 \!=\!\{a_1,a_2,a_4\}$, and thus $\theta^* = \{a_1,a_2,a_3,a_4\}$.

The polyhedron resulting from $\theta_1$ is:
\begin{align*}
u_1 +  u_3 & > u_4 \\
u_1 + u_3 & > 0 \\
u_1 & > 0 \\
u_4 & > 0 \\
u_4 & > u_3 \\
0 & > u_3
\end{align*}
%$\{u_1 +  u_3  > u_4; u_1 + u_3 > 0; u_1  > 0; u_4 > 0;  u_4 > u_3; 0 > u_3\}$.

From $u_3\!<\!0 $ and $u_1 + u_3\!>\!0$ it results that $u_1\!>\!0$ and since $u_2\!=\!0$ because $a_2\!\not\in\!\theta_1$ we have
$$(1, 0, 0, 0)\!\succ_{\theta_1}^R\!(0, 1, 0, 0).$$

The polyhedron resulting from $\theta_2$ is: 
\begin{align*}
u_1 +  u_2 & > u_4 \\
u_1 + u_2 & > 0 \\
u_1 & > 0 \\
u_4 & > 0 \\
u_4 & > u_2 \\
0 & > u_2
\end{align*}
%$\{u_1 +  u_2  > u_4; u_1 + u_2 > 0; u_1  > 0; u_4 > 0; u_4 > u_2; 0 > u_2 \}$. 

From $u_2\!<\!0$ and $u_1\!>\!0$ we have
$$(1,0,0,0) \succ_{\theta_2}^R (0,1,0,0).$$

%We conclude that $((1,0,0,0), (0,1,0,0))\!\in\!\cap_{\theta_i \in \Theta_{min}} R^{\theta_i}$.

Hence, $(1,0,0,0)$ is strictly preferred to $(0,1,0,0)$ for both $\succ_{\theta_1}^R$ and $\succ_{\theta_2}^R$. 

Yet, the polyhedron resulting from $\theta^*$ is 
\begin{align*}
u_1 +  u_2 + u_3 & > u_4 \\
u_1 + u_2 + u_3 & > 0 \\
u_1 & > 0 \\
u_4 & > 0 \\
u_4 & > u_2 + u_3\\
0 & > u_2 + u_3
\end{align*}

%$\{u_1 + u_2 + u_3 > u_4 ;
%u_1 + u_2 + u_3 > 0 ;
%u_1 + u_2 + u_3 > u_2 + u_3 ;
%u_4 > 0 ;
%u_4 > u_2 + u_3; 
%0 > u_2 + u_3 \}$.

And we can verify that $$u=\{u_1 = 3; u_2 = 5; u_3 = -6; u_4 = 1\} \in U_R^{\theta^*}$$ and since $f_{\theta^*,u}(0,1,0,0) > f_{\theta^*,u}(1,0,0,0)$ the preference $(1,0,0,0)\!\succ_{\theta^*}^R\!(0,1,0,0)$ does not hold while $(1,0,0,0)$ is strictly preferred to $(0,1,0,0)$ for both $\succ_{\theta_1}^R$ and $\succ_{\theta_2}^R$.
\end{example}
\end{document}